\pdfoutput=1

\documentclass[11pt]{article}
\usepackage[dvipsnames]{xcolor}

\usepackage[acceptedWithA]{tacl2021v1}

\usepackage{times}
\usepackage{latexsym}
\usepackage{xfrac}

\usepackage[font=small]{caption}

\usepackage{tikz}
\usetikzlibrary{bayesnet}

\usepackage[T1]{fontenc}

\usepackage[utf8]{inputenc}

\usepackage{microtype}

\usepackage{inconsolata}

\usepackage{booktabs}

\usepackage{rycolab}
\usepackage{subspaceprobing}

\title{A Geometric Notion of Causal Probing}

\author{Cl{\'e}ment Guerner\quad Tianyu Liu \quad Anej Svete\quad Alexander Warstadt \quad Ryan Cotterell \\
\setlength{\fboxsep}{2.5pt}%
\setlength{\fboxrule}{2.5pt}%
\fcolorbox{white}{white}{
    $\{$\texttt{\href{mailto:cguerner@inf.ethz.ch}{cguerner}}, \texttt{\href{mailto:tianyu.liu@inf.ethz.ch}{tianyu.liu},
    }\texttt{\href{mailto:anej.svete@inf.ethz.ch}{anej.svete}, }\texttt{\href{mailto:awarstadt@inf.ethz.ch}{awarstadt},
    }\texttt{\href{mailto:ryan.cotterell@inf.ethz.ch}{ryan.cotterell}}$\}$\texttt{@inf.ethz.ch}
} \\
    {%
\setlength{\fboxsep}{2.5pt}%
\setlength{\fboxrule}{2.5pt}%
\fcolorbox{white}{white}{
    \includegraphics[width=.15\linewidth]{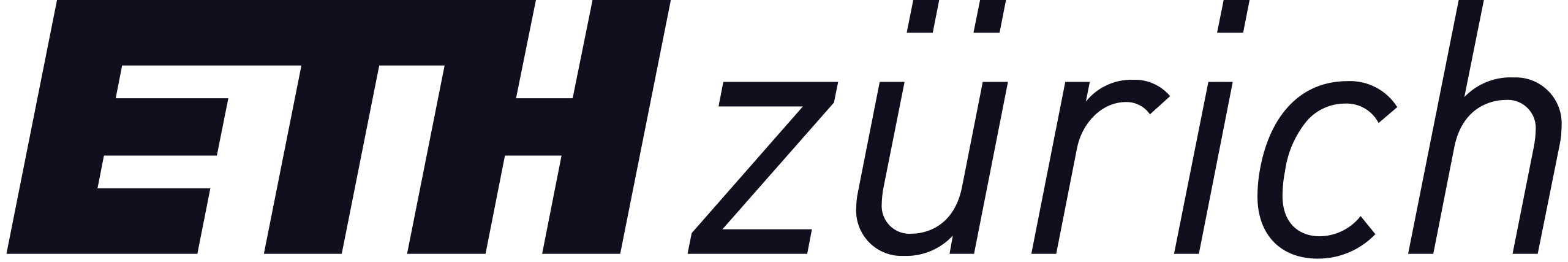}
}
}}

\begin{document}
\maketitle

\begin{abstract}
The linear subspace hypothesis \citep{bolukbasi2016} states that, in a language model's representation space, all information about a concept such as verbal number is encoded in a linear subspace.
Prior work has relied on auxiliary classification tasks to identify and evaluate candidate subspaces that might give support for this hypothesis.
We instead give a set of intrinsic criteria which characterize an ideal linear concept subspace and enable us to identify the subspace using only the language model distribution.
Our information-theoretic framework accounts for spuriously correlated features in the representation space \citep{kumar2022probing} by reconciling the statistical notion of concept information and the geometric notion of how concepts are encoded in the representation space. 
As a byproduct of this analysis, we hypothesize a causal process for how a language model might leverage concepts during generation. 
Empirically, we find that linear concept erasure is successful in erasing most concept information under our framework for verbal number as well as some complex aspect-level sentiment concepts from a restaurant review dataset.
Our causal intervention for controlled generation shows that, for at least one concept across two languages models, the concept subspace can be used to manipulate the concept value of the generated word with precision.
\newline
\vspace{0.1em}

\hspace{.1em}\raisebox{-0.2em}{\includegraphics[width=1.05em,height=1.05em]{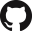}}%
\hspace{.3em}%
\parbox{\dimexpr\linewidth-2\fboxsep-2\fboxrule}{\small \url{https://github.com/rycolab/causalgeom}}
\end{abstract}

\section{Introduction} \label{sec:intro}

The reliance of language models (LMs) on concepts to make predictions---especially linguistic concepts such as \conceptf{verbal-number}\footnote{Throughout the text, we will use a distinguished typesetting to refer to concepts. For instance, the concept of a bird is written as \conceptf{bird}.}---is a well-studied phenomenon \citep{ravfogel-etal-2021-counterfactual,lasri-etal-2022-probing, amini-etal-2023-causal, arora-etal-2024-causalgym}. Earlier studies on this topic test whether an LM uses the concept of \conceptf{verbal-number} by giving it a forced choice between a grammatical and an ungrammatical variant of a sentence \cite{linzen2016assessing, marvin-linzen-2018-targeted, goldberg2019assessing, lasri-etal-2022-probing}. Consider, for example, the sentences: \vspace{-5pt}

\ex. \a. \label{ex:gramm} \textit{The kids \textbf{walk} the dog.}\\ 
        {\small the kid.\textsc{pl} walk.\textsc{3pl.pres} the dog.\textsc{sg}}\\
     \b. \label{ex:ungramm} $^*$\textit{The kids \textbf{walks} the dog.}\\
        {\small the kid.\textsc{pl} walk.\textsc{3sg.pres} the dog.\textsc{sg}}\\[-15pt]
\par

\citet{goldberg2019assessing} shows that LMs can achieve near perfect accuracy when forced to choose between two such variants. This result suggests that LMs make use of \conceptf{verbal-number} and other concepts to perform next-word prediction, but tells us little about how the representation spaces of these models encode such concepts.

Our primary contribution is to construct a novel geometric notion of what it means for a neural LM's representation space\footnote{For now, we define a representation space simply as the $d$-dimensional vector space that a language model relies on to encode text. We propose a more formal definition in \cref{sec:lmsconcepts}.} to have information about a concept. 
Following \citet{bolukbasi2016} and \citet{ravfogel2022linear}, we argue that concepts are naturally operationalized by \emph{linear} subspaces. 
Linear subspaces lend themselves to tractable algorithms, and they have a simple geometric interpretation which makes it possible to erase a concept from a representation.
Existing work \citep{lasri-etal-2022-probing, ravfogel-etal-2023-linear} has relied on $\mathcal{V}$-information \cite{xu2020theory} to quantify the amount of information in the representation space of a language model, before and after concept erasure. 
This measure is \emph{extrinsic} to the language model, in the sense that it relies on a variational family $\mathcal{V}$ of auxiliary classifiers to measure concept information.
In contrast, we propose an \emph{intrinsic}, information-theoretic \citep{shannon1948mathematical} definition of information, by which we mean that information is quantified using distributions induced from the language model, i.e., without relying on an additional classifier.\looseness=-1

We show, via an example inspired by \citet{kumar2022probing}, that a naïve approach to measuring intrinsic information in a subspace falls victim to spurious correlations. 
Specifically, while a ground truth, \emph{causal} concept subspace may exist in the representation space, correlated non-concept features can also contain information about the concept, complicating the task of estimating concept information in either subspace.
Our framework breaks the dependence between the concept subspace and its orthogonal complement, allowing us to \emph{correctly} compute information contained in either subspace while marginalizing out the other.
This approach is counterfactual in the sense that it creates representations that would not otherwise occur under the language model.
Crucially, it allows us to talk about the mutual information between linear subspaces and concepts.\looseness=-1

We derive four geometric properties within our counterfactual framework that characterize a precise geometric encoding of a concept.
First, \textbf{erasure} is the condition that the orthogonal complement of the concept subspace should contain \emph{no} information about the concept. 
Second, \textbf{encapsulation} states that projecting a representation onto our concept subspace should preserve \emph{all} the information about the concept.
Third, \textbf{stability} quantifies the requirement that projection onto the orthogonal complement of our concept subspace should preserve non-concept information.
Finally, \textbf{containment} ensures that the concept subspace does not contain additional information beyond the concept.\looseness=-1

Empirically, we study linguistic concepts \conceptf{verbal-number} in English and \conceptf{grammatical-gender} in French.  
We find, for \conceptf{verbal-number}, that the LEACE method for linear concept erasure \citep{belrose2023leace} yields a one-dimensional concept subspace which, according to our novel counterfactual metrics, contains a large share of concept information while leaving non-concept information relatively untouched. 
We then leverage our intrinsic measure of information to posit a causal graphical model by which a latent concept may govern LM text generation.
This model enables us to derive a causal controlled generation method by manipulating the concept component of a representation.
And, indeed, we find evidence that it is possible to use a one-dimensional subspace to control the generation behavior of two language models with respect to \conceptf{verbal-number}, but not for \conceptf{grammatical-gender}.
We test our causal intervention against the CEBaB \citep{abraham-etal-2022-cebab} benchmark, and find our do-intervention improves the performance of INLP \citep{ravfogel-etal-2020-null} and LEACE as causal effect estimators.

\section{Concepts and Information} \label{sec:lmsconcepts} 

In this section, we build towards a definition of mutual information between representations and the concept of interest.
\looseness=-1

\subsection{Language Modeling Basics} \label{sec:lms}

A language model is a probability distribution $\plm$ over $\kleene{\alphabet}$, the Kleene closure over an alphabet $\alphabet$.
We parameterize $\plm$ in an autoregressive manner \citep{du-etal-2023-measure} as follows: \vspace{-10pt} 
\begin{equation}\label{eq:lm}
\plm(\str) = \plm(\eos \mid \str) \prod_{t=1}^{T} \plm(\wordt \mid \cxt) \vspace{-5pt}
\end{equation}
where $\wordt \in \alphabet$ refers to $t$-th word\footnote{We refer to $\word \in \alphabet$ as words for simplicity, even though in the context of language modeling, these are often called subwords, tokens, or symbols.} in a string $\str \in \kleene{\alphabet}$, $\cxt$ represents the first $t-1$ words of $\str$, and $\eos \notin \alphabet$ being a distinguished end-of-string symbol.

Many language models make use of contextual representations, i.e., they encode a textual context $\cxt$ as a real-valued column vector $\cxtenc \in \Rd$.
Generally, $\cxtenc$ is deterministically computed from the context string $\cxt$\footnote{We relax this assumption later on, such that $\cxtenc$ can be stochastic given $\cxt$. One example of a language model with stochastic contextual embeddings is \citet{bowman-etal-2016-generating}.\looseness=-1}, such that the representation space of \cref{eq:lm} is defined as
\begin{equation}
\repspace \defeq \Big\{ \bh(\str) \mid \str \in \kleene{\alphabet} \Big\} \subseteq \Rd.\footnote{\textnormal{Despite consisting of real vectors, the cardinality of $\repspace$ is \emph{countably} infinite, because it contains exactly one element for every string in the countably infinite set $\kleene{\alphabet}$.
Thus, summing over $\repspace$\ is discrete and does not require integration.}}
\end{equation}

\subsection{Language Models and Concepts} \label{sec:concepts}

We now discuss an exact sense in which a language model can be said to encode a concept. 
First, we define a concept based on the possible values it can take.
We formalize this with a \defn{concept set}, a finite, non-empty set $\concepts$ whose elements are those values.
For example, we take the concept set for \conceptf{verbal-number} to include three values: \conceptval{sg} (e.g., \wordf{walks}), \conceptval{pl} (e.g., \wordf{walk}), and \conceptval{n/a} (e.g., \wordf{consternation}). 
For various reasons, including syncretism \citep{baerman2007syncretism}, some verbs in English can have ambiguous concept value depending on context.
For instance, in the sentence \wordf{You walked to the store}, \wordf{walked} can be \conceptval{sg} or \conceptval{pl}.
We find similar facts for other concepts in different languages, e.g., for \conceptf{grammatical-gender} in French, the adjective \wordf{marron} can be both \conceptval{fem} and \conceptval{msc}.\looseness=-1

To relate language models to concept sets, we introduce a probability distribution $\iota(\concept \mid \cxt, \word)$. 
$\iota$ tells us the probability that, in the sequential context $\cxt \in \kleene{\alphabet}$, word $\word \in \alphabet$ is annotated with the concept value $\concept \in \concepts$.
For now, we make a simplifying assumption that $\iota$ is deterministic, i.e., $\iota(\concept \mid \cxt, \word) \in \{0, 1\}$ for all $\concept \in \concepts$, $\word \in \alphabet$, and $\cxt \in \kleene{\alphabet}$.\footnote{To illustrate this formalism, consider the concept \conceptf{verbal-number} and \cref{ex:gramm,ex:ungramm}. The concept set for \conceptf{verbal-number} is $\concepts = \{\conceptvalfootnote{sg}, \conceptvalfootnote{pl}, \conceptvalfootnote{n/a}\}$, and $\iota$ maps as follows, e.g., $\iota(\conceptvalfootnote{sg} \mid \wordf{The kids}, \wordf{walk}) = 0$, $\iota(\conceptvalfootnote{pl} \mid \wordf{The kids}, \wordf{walk}) = 1$.
}
This assumption will be relaxed in \cref{sec:causal} with a stochastic operationalization of concepts.

\looseness=-1




\subsection{Unigram Information}\label{sec:reps-information}

To construct a mutual information between the model's notion of a concept and its contextual representations, we require a joint distribution between a concept-valued random variable and a representation-valued random variable. In order for this estimate to be intrinsic, we obtain this distribution from the language model itself.\looseness=-1

We begin by defining the \defn{joint induced unigram} distribution of the language model over words, concepts, and representations in \cref{eq:3-var-joint-unigram}. 
In words, this distribution tells how frequently each word $\word \in \alphabet$ co-occurs with a concept value $\concept \in \concepts$ and a representation $\bh \in \repspace$, on average, in a string $\str \sim \plm$ of length $T$:
\vspace{-15pt}
\begin{align} \label{eq:3-var-joint-unigram}
    &\punigram(\word, \concept, \bh) \defeq 
    \sum_{\str \in \kleene{\alphabet}} \plm (\str)   \\[-5pt]
    & \;\; \frac{\sum_{t=1}^{T}  \iota(\concept \mid \cxt, \wordt) \mathbbm{1}\Big\{ \word = \wordt \land
    \bh = \cxtenc \Big\}}{T} \nonumber
\end{align}

We can now use \cref{eq:3-var-joint-unigram} to compute our intrinsic measure of concept information in representations:\looseness=-1
\begin{equation} \label{eq:cxtconceptmi}
\MI(\rvC ; \rvH) =
\sum_{\concept \in \concepts}
\sum_{\bh \in \repspace} \punigram(\concept, \bh) \log \!\frac{\punigram(\concept, \bh)}{\punigram(\concept) \punigram(\bh)}
\end{equation}
where $\rvC$ is a $\concepts$-valued random variable, $\rvH$ is a $\repspace$-valued random variable, and $\punigram(\concept, \bh)$ is obtained by marginalizing out $\word$ from \cref{eq:3-var-joint-unigram}. 
\Cref{eq:cxtconceptmi} tells us how much information on average a representation $\bh \in \repspace$ encodes about the identity of a concept $\concept \in \concepts$. \looseness=-1


Next, we define the following conditional mutual information: 
\begin{align}\label{eq:conditionalmi}
\MI(&\rvX ; \rvH \mid \rvC) = \\[-5pt]
&\sum_{\concept \in \concepts}
\sum_{\word \in \alphabet} \sum_{\bh \in \repspace} 
\punigram(\word, \bh, \concept) \log \frac{\punigram(\word, \bh \mid \concept) }{\punigram(\word\mid \concept) \punigram(\bh \mid \concept)} \nonumber
\end{align}
This quantity measures, given a particular concept value $\concept \in \concepts$, how much additional information about a word $\word \in \alphabet$ is encoded in the model's representations. 
Our information-theoretic framework can be generalized to handle different language-generating processes, e.g., different decoding algorithms for language models, or natural text generated by a process other than the language model under study.\looseness=-1

\section{A Geometric Encoding of Concepts} \label{sec:functionalsubspace}

The \defn{linear subspace hypothesis} \cite{bolukbasi2016} makes a prediction about how the concept information we quantify in the previous section is represented geometrically in the LM's representation space. 
Specifically, it postulates that there exists a \emph{linear subspace} $\linspacec \subseteq \Rd$
that contains all of the information about a concept with values $\concepts$.\footnote{Note that $\linspacec$ is \emph{not} a linear subspace in the linear-algebraic sense because it is countable and thus not closed under scalar multiplication. 
}
This hypothesis has been tested on various linguistic concepts, including \conceptf{verbal-number} \citep{ravfogel-etal-2021-counterfactual,lasri-etal-2022-probing, amini-etal-2023-causal} and \conceptf{grammatical-gender} \citep{amini-etal-2023-causal}.
We follow in this vein, and decompose the representation space $\repspace$ into a concept linear subspace and an orthogonal, non-concept linear subspace.
Then, we provide four information-theoretic metrics that characterize these subspaces in terms of the information that they contain.

\subsection{Concept Partition} \label{sec:concept-partition}


Given a concept set $\concepts$, we define a partition of a language model's representation space $\repspace$ into a \defn{concept subspace} $\linspacec$ and its orthogonal complement, the \defn{non-concept subspace} $\linspaceccomp$. 
We refer to $\bPk \in \R^{d \times d}$ as the orthogonal projection matrix that projects onto $\linspaceccomp$, i.e., $\bPk \bh = \text{proj}_{\linspaceccomp}(\bh)$.
In turn, $\eyeminusPk$ projects onto the concept subspace $\linspacec$ with dimensionality $|\concepts|-1$, such that $(\eyeminusPk) \bh = \text{proj}_{\linspacec}(\bh)$.
We refer to the partition of $\Rd$ into $\linspacec$ and $\linspaceccomp$ as an \defn{information partition}.\looseness=-1

We use \cref{eq:cxtconceptmi} to define the information about the concept encoded in both. 
Consider, for example, information in $\linspaceccomp$ about $\concepts$ on average over textual contexts:\vspace{-5pt}
\begin{align} \label{eq:subspace-mi}
\MI(\rvC ;& \bPk \rvH) =
\\ &\sum_{\concept \in \concepts}
\sum_{\bh \in \repspace} 
\punigram(\concept, \bPk \bh) \log \!\frac{\punigram(\concept, \bPk \bh)}{\punigram(\concept) \punigram(\bPk \bh)} \nonumber
\end{align}
where the language model's representations are orthogonally projected onto $\linspaceccomp$ using $\bPk$. 
\cref{eq:subspace-mi} relates the \emph{geometric} notion of a linear subspace with the \emph{information-theoretic} notion of information.
Thus, if $\MI(\rvC ; \bPk \rvH)$ is low, we can say that $\bPk$ erases a lot of concept information in $\repspace$ by projecting onto the subspace $\linspaceccomp$.
We denote $\repspacepar \defeq \{ \eyeminusPkh \mid \bh \in \repspace \}, \repspacebot \defeq \{ \bPk  \bh \mid \bh \in \repspace \}$, and refer to $ \rvHconcept, \rvHerase$ as random variables corresponding to contextual representations projected onto concept and non-concept subspaces, respectively.\looseness=-1

\begin{table}[t!]
    \centering
    \begin{minipage}[b]{1.0\columnwidth}    
        \centering
        \includegraphics[width=0.8\columnwidth,trim={2mm 0mm 2mm 0},clip]{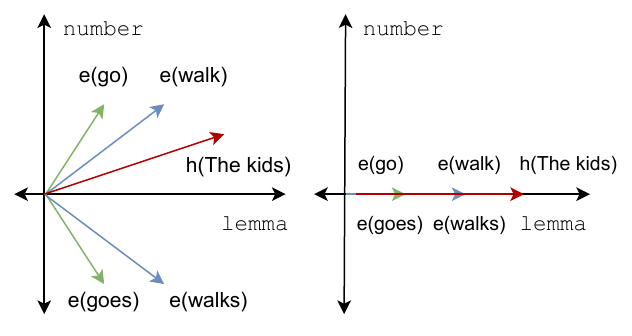}
            \captionof{figure}{Example of erasure of a \conceptf{verbal-number} subspace, when predicting the next word given \wordf{The kids}. The representation space is two-dimensional with the $y$-axis representing the correct subspace encoding the concept $\conceptf{verbal-number}$, while the $x$-axis encodes the lemma. Word representations are denoted with $\mathbf{e}$ and contextual representation with $\bh$. On the left, we have the original representation space, and on the right, we have the space resulting from erasing information in our concept subspace, i.e., setting the $y$-coordinates of all vectors in the space to 0.\looseness=-1}\label{fig:geomexample}
    \end{minipage}\vspace{5pt}
    \begin{minipage}[b]{1.0\columnwidth}
        \captionsetup{type=table}
        \begin{subtable}{\columnwidth}
            \small
            \centering
            \begin{tabular}{ ccccc } 
                 \toprule
                  & \wordf{walks} & \wordf{walk} & \wordf{goes} & \wordf{go} \\
                 \midrule
                 \conceptval{sg} & 0 & 0 & 0.7 & 0  \\ 
                 \conceptval{pl} & 0 & 0.3 & 0 & 0 \\ 
                 \bottomrule
            \end{tabular}
        \end{subtable} 
        \caption{Hypothetical joint unigram distribution $\punigram(\word, \concept)$ of \conceptf{verbal-number} and word. The lemma \wordf{walk} is only used as \conceptvalfootnote{pl} and \wordf{go} only as \conceptvalfootnote{sg}.} \label{tab:counterexample-joint}
        \vspace{-15pt}
    \end{minipage}
\end{table}

\subsection{The Perils of Correlation} \label{sec:example}

\Cref{eq:subspace-mi} suggests an attractive property we might ask from $\bPk$: It should satisfy $\MI\left(\rvC; \bPk \rvH\right) = 0$, i.e., completely erase the information about the concept by projecting onto $\linspaceccomp$.
However, as we show next, this na{\"i}ve characterization is flawed.
We illustrate this point with a counterexample inspired by \citet{kumar2022probing}, shown in \cref{fig:geomexample}.
Intuitively, such a transformation constitutes successful erasure.\footnote{One might, but probably shouldn't, refer to the y-axis as the \emph{causal} subspace, in the sense that manipulating the value in that subspace would result in changing precisely the concept encoded by the representation while leaving other aspects intact.}
To the extent that such a subspace exists in reality, finding the $\bPk$ that erases this subspace seems like the correct objective.

Now, consider the hypothetical joint word--concept unigram distribution $\punigram(\word, \concept)$ in \Cref{tab:counterexample-joint}. 
Under this distribution, a projection matrix $\bPk$ that erases the correct $y$-axis as shown in \cref{fig:geomexample} is \emph{not} the minimizer of \cref{eq:subspace-mi}. 
Knowledge of the lemma alone reveals the \conceptf{verbal-number}, because $\rvHerase$ ($x$-axis) and $\rvHconcept$ ($y$-axis) are heavily correlated.
This means that $\MI(\rvC; \rvHerase) = 0.88 > 0$ in our toy example in \cref{fig:geomexample}. 
In order to have $\MI(\rvC; \rvHerase) = 0$, we would need to let $\bPk = \mathbf{0}$, thereby erasing all lemma information as well.
Thus, requiring $\bPk$ to satisfy $\MI(\rvC; \rvHerase) = 0$ does not characterize successful erasure because it requires removing all spuriously correlated features.\looseness=-1

\subsection{A Counterfactual Unigram Distribution} \label{sec:counterfactual}

The underlying problem with the example given in \cref{sec:example} is that $\rvHconcept$ and $\rvHerase$ have a common cause that introduces a spurious correlation---the $\kleene{\alphabet}$-valued context random variable $\rvbX$.
This means $\MI(\rvHerase; \rvHconcept) > 0$, i.e., $\rvHerase$ and $\rvHconcept$ are \emph{not} statistically independent. 
We resolve this issue by building a variant of our information-theoretic objective in \cref{eq:subspace-mi} that \emph{assumes} these two variables are statistically \emph{independent}, i.e., $\MI(\rvHerase; \rvHconcept) = 0$.
Under this assumption, $\rvHerase$ would contain no information about the concept, and identification of $\rvHconcept$ would be possible via mutual information.
While this assumption likely never holds for a concept in practice, this does not matter here---we are crafting a metric under which the correct subspace will be optimal.\looseness=-1

We denote with $\bhpar \defeq (\eyeminusPk) \bh$ and $\bhbot \defeq \bPk \bh$ the projections onto the concept and non-concept subspace for $\bh \in \repspace$. 
Marginalizing with respect to the induced unigram distribution defined in \cref{sec:lmsconcepts}, we arrive at the following unigram distributions: 
\begin{align}\label{eq:counterfactual-unigram-dist-bot}
    \punigram(\bhbot) &\defeq \sum_{\bh \in \repspace}  \mathbbm{1}\{\bhbot = \bPk \bh\} \punigram(\bh) \\
\label{eq:counterfactual-unigram-dist-par}     
    \punigram(\bhpar) &\defeq \sum_{\bh \in \repspace} \mathbbm{1}\{\bhpar = (\eyeminusPk) \bh\} \punigram(\bh)     
\end{align}
We now construct a variant of our induced unigram $\punigram(\word, \concept, \bh)$ that assumes independence between $\bhbot$ and $\bhpar$, i.e., $\qunigram(\bh) = \qunigram(\bhbot, \bhpar) \defeq \punigram(\bhbot)\,\punigram(\bhpar)$.
This \defn{counterfactual unigram distribution} $\qunigram$ assigns probability mass to $(\bhbot, \bhpar)$ pairs which, under $\punigram(\bh)$, would have zero probability.\looseness=-1
\begin{align}\label{eq:counterfactualunigram}
  \qunigram(\word, &\concept,\bhpar, \bhbot) \defeq \sum_{\cxt \in \kleene{\alphabet}}\,\iota(\concept \mid \word, \cxt) \\
&\plm(\word \mid \bhpar, \bhbot)\,\plm(\cxt)\, \punigram(\bhpar)\, \punigram(\bhbot) \nonumber
\end{align}
The choice of the name counterfactual, as well as the implications of this decoupling, will be made precise in \cref{sec:causal} when we introduce the causal interpretation of the word--concept model.\looseness=-1

We define the \defn{counterfactual mutual information} between the concept and the projection onto the non-concept subspace as\looseness=-1
\begin{align}  \label{eq:miqconceptcounterfactual}
    \MIq (& \rvC; \rvHerase ) \defeq \\
    &\sum_{\concept \in \concepts} \sum_{\bhbot \in \repspacebot} 
    \qunigram(\concept, \bhbot) \log \!\frac{\qunigram(\concept, \bhbot)}{\qunigram(\concept) \qunigram(\bhbot)} \nonumber
\end{align}
Importantly, \cref{eq:miqconceptcounterfactual} is minimized by the correct subspace in our example in \cref{sec:example}. 
Note that $\MIq (\rvC; \rvHconcept)$ can also be obtained by marginalizing out $\bhbot$ instead.
Finally, we define $\MIq (\rvX ; \rvHerase \mid \rvC)$ by using $\qunigram$ instead of $\punigram$ in \cref{eq:conditionalmi}.\looseness=-1

\subsection{Erasure and Encapsulation} \label{sec:erasure}

We now give formal definitions of erasure and encapsulation based on \cref{eq:miqconceptcounterfactual}. 
These two notions, combined, determine the extent to which a projection matrix $\bPk$ has decomposed the representation space into concept and non-concept subspaces.

\begin{defin}[Counterfactual Erasure]\label{def:eraser}
Let $\rvHerase \defeq \bPk \rvH$ be an $\Rd$-valued random variable.
An orthogonal projection matrix $\bPk \in \R^{d \times d}$ is an $\varepsilon$-\defn{eraser} of $\concepts$ if $\MIq(\rvC; \rvHerase) < \varepsilon$.\looseness=-1
\end{defin}
As $\varepsilon \to 0$, the subspace $\linspaceccomp$ characterized by an $\varepsilon$-eraser $\bPk$ for concept set $\concepts$ with respect to $\repspace$ encodes very little information about the concept.
This means that the language model is no longer able to determine the concept value required by the textual context when generating the next word.
We now show that given an $\varepsilon$-eraser $\bPk$, projecting onto its orthogonal complement with $\eyeminusPk$ preserves nearly all of the information.\looseness=-1

\begin{defin}[Counterfactual Encapsulation]\label{def:encapsulation}
Let $\rvHconcept \defeq \left(\eyeminusPk\right) \rvH$ be an $\Rd$-valued random variable.
An orthogonal projection matrix $\eyeminusPk \in \R^{d \times d}$ is an $\varepsilon$-\defn{encapsulator} 
of $\concepts$ if $\MIq(\rvC; \rvH) - \MIq(\rvC; \rvHconcept) < \varepsilon$.\looseness=-1
\end{defin}
The quantity $\MIq(\rvC; \rvH) - \MIq(\rvC; \rvHconcept)$ is always non-negative due to the data-processing inequality \citep[\S2.8]{coverthomas2006}.
Encapsulation operationalizes the idea that a subspace gives us all the information needed to correctly identify the concept value required by textual context.
In \cref{app:pleasantdecomposition}, we show that the mutual information can be additively decomposed: $\MIq(\rvC; \rvH) = \MIq(\rvC; \rvHerase) + \MIq(\rvC; \rvHconcept)$. 

\newcommand{\minnodesize}{6.5ex}
\begin{figure*}[ht!]
   \centering
    \begin{adjustbox}{width=0.8\textwidth}
    \begin{subfigure}{0.10\textwidth}
    \centering
    \begin{tikzpicture}[latent/.style={circle, draw, minimum size=\minnodesize}, 
                        obs/.style={circle, draw, fill=lightgray, minimum size=\minnodesize}, node distance=0.7cm]
        \node[latent]   (word) {$\rvX$};        
        \node[det, above=of word]   (repspace) {$\rvH$};
        \node[latent, above=of repspace]   (context) {$\rvbX$};
        
        \edge {context} {repspace} ; %
        \edge {repspace} {word} ; %
    \end{tikzpicture}
    \caption{} \label{fig:causal-graph-a}
    \end{subfigure}
    \hspace{1ex}%
    \begin{subfigure}{0.29\textwidth}
    \centering
    \begin{tikzpicture}[latent/.style={circle, draw, minimum size=\minnodesize}, 
                        obs/.style={circle, draw, fill=lightgray, minimum size=\minnodesize}]
        
        \node[latent]   (word) {$\rvX$};        
        \node[det, above=0.5cm of word]   (repspace) {$\rvH$};
        \node[det, above right=1.3cm of repspace, yshift=-5.5mm] (conceptsubspace) {$\rvHconcept$};
        \node[det, above left=1.3cm of repspace, yshift=-5.5mm]     (compsubspace) {$\rvHerase$};
        \node[const, left=of conceptsubspace, xshift=2.75mm]     (bp) {$\bPk$};
        \node[latent, above=0.5cm of compsubspace]   (context) {$\rvbX$};
        \node[latent, above=0.5cm of conceptsubspace]   (concept) {$\rvC$};        
        
        \edge[color=blue] {context} {compsubspace, concept} ; %
        \edge {context} {conceptsubspace} ; %
        \edge {conceptsubspace} {repspace} ; %
        \edge[color=blue] {compsubspace} {repspace} ; %
        \edge {bp} {compsubspace, conceptsubspace} ; %
        \edge[] {concept} {conceptsubspace} ; %
        \edge {repspace} {word} ; %

    \end{tikzpicture}
    \caption{} 
    \label{fig:causal-graph-b}
    \end{subfigure}
    \hspace{4ex}%
    \begin{subfigure}{0.29\textwidth}
    \centering
    
    \begin{tikzpicture}[latent/.style={circle, draw, minimum size=\minnodesize}, 
                        obs/.style={circle, draw, fill=lightgray, minimum size=\minnodesize}]
        
        \node[latent]   (word) {$\rvX$};        
        \node[det, above=0.5cm of word]   (repspace) {$\rvH$};
        \node[det, above right=1.3cm of repspace, yshift=-5.5mm] (conceptsubspace) {$\rvHconcept$};
        \node[det, above left=1.3cm of repspace, yshift=-5.5mm]     (compsubspace) {$\rvHerase$};
        \node[const, left=of conceptsubspace, xshift=2.75mm]     (bp) {$\bPk$};
        \node[latent, above=0.5cm of compsubspace]   (context) {$\rvbX$};
        \node[latent, fill=red!30!white, above=0.5cm of conceptsubspace]   (concept) {$\rvC$};        
        
        \edge {context} {compsubspace, conceptsubspace} ; %
        \edge {conceptsubspace} {repspace} ; %
        \edge {compsubspace} {repspace} ; %
        \edge {bp} {compsubspace, conceptsubspace} ; %
        \edge {concept} {conceptsubspace} ; %
        \edge {repspace} {word} ; %
    \end{tikzpicture}
    \caption{}
    \label{fig:causal-graph-c}
    \end{subfigure}
    \hspace{2ex}%
    \begin{subfigure}{0.29\textwidth}
    \centering
    
    \begin{tikzpicture}[latent/.style={circle, draw, minimum size=1.2cm}, 
                        obs/.style={circle, draw, fill=lightgray, minimum size=1.2cm}]
        \node[latent]   (context) {$\rvbX$};
        \node[det, below=0.5cm of context]   (repspace) {$\rvH$};
        \node[det, below right=1.2cm of repspace, yshift=5.mm] (conceptsubspace) {$\rvHconcept$};
        \node[det, below left=1.2cm of repspace, yshift=5.mm]     (compsubspace) {$\rvHerase$};
        \node[const, left=of conceptsubspace, xshift=2.75mm]     (bp) {$\bPk$};
        \node[latent, below=0.3cm of conceptsubspace, yshift=-1.5mm]    (concept) {$\rvC$};
        \node[latent, below=0.3cm of compsubspace, yshift=-1.5mm] (x) {$\rvX$};

        \edge {context} {repspace} ; %
        \edge {repspace} {conceptsubspace,compsubspace} ; %
        \edge {conceptsubspace} {concept} ; %
        \edge {compsubspace,conceptsubspace} {x} ; %
        \edge {bp} {conceptsubspace,compsubspace}
    \end{tikzpicture}
    \caption{} \label{fig:causal-graph-d}
    \end{subfigure}
    \end{adjustbox}
    \caption{Causal graphical models that demonstrate how a concept may have a causal effect on word generation. Circles represent random variables and diamonds represent deterministic variables. $\rvbX, \rvC, \rvX$ represent the random variables for the textual context, the underlying concept, and the next word, respectively. $\rvH, \rvHconcept, \rvHerase$ are the representation at step $t$, its concept-related component, and its component whose concept-related information is erased by orthogonal projection matrix $\bPk$. \Cref{fig:causal-graph-a} shows the traditional autoregressive causal structure for generation. \Cref{fig:causal-graph-b} is our proposed causal structure for generation with a $\concepts$-valued latent variable $\rvC$, with the \backdoorcolor{backdoor path} from $\rvC$ to $\rvH$ shown in blue. \Cref{fig:causal-graph-c} is the causal structure induced by a do-intervention on $\rvC$. Finally, \cref{fig:causal-graph-d} is the causal structure implied by \citeposs{yang-klein-2021-fudge} concept-controlled generation approach.} \vspace*{-5mm} \label{fig:causal-graph}
\end{figure*}

\subsection{Containment and Stability} \label{sec:containstab}

Erasure and encapsulation do not consider the information content of the representation aside from the concept.
With perfect erasure and encapsulation, the learned orthogonal projection matrix $\bPk$ could erase much of the non-concept related information from $\linspaceccomp$. 
Specifically, if $\concepts$ is encoded non-linearly \citep{ravfogelkernel2022}, then erasure via a linear orthogonal projection could require the removal of additional dimensions that also contain non-concept information.
Therefore, in the concept erasure literature, tests of successful erasure are paired with a verification that the representations are not otherwise damaged \citep{kumar2022probing,ravfogel-etal-2020-null,ravfogel2022linear,ravfogelkernel2022,elazar-etal-2021-amnesic}. 
We, too, need an information-theoretic notion of preservation of non-concept information in $\rvHerase$.

Preserving information about non-concept aspects of $\cxt$ in $\rvHerase$ requires that $\rvHconcept$ \emph{only} capture information about the concept, i.e. that it should be the \emph{minimal} subspace that captures $\rvC$. 
Containment formalizes this notion by requiring that, conditioned on $\rvC$, $\rvHconcept$ contains little information about the next word $\rvX$.\looseness=-1
\begin{defin}[Counterfactual Containment] \label{def:containment}
Let $\bPk$ be an eraser for concept set $\concepts$ with respect to $\repspace$.
Let $\rvHconcept \defeq \left(\eyeminusPk\right) \rvH$ be an $\Rd$-valued random variable.
Then, we say that $\bPk$ is $\varepsilon$-\defn{contained} with respect to $\repspace$ and $\concepts$ if $\MIq(\rvX; \rvHconcept \mid \rvC) < \varepsilon$.\looseness=-1
\end{defin}
Lastly, we define stability to measure how much non-concept information about the next word is \emph{preserved} in the non-concept subspace $\rvHerase$.
Ideally, this should be as close as possible to the information present in the entire representation space, ignoring the information about the concept.
\begin{defin}[Counterfactual Stability] \label{def:stability}
Let $\bPk$ be an eraser for concept set $\concepts$ with respect to $\repspace$.
Let $\rvHerase \defeq \bPk \rvH$ be an $\Rd$-valued random variable.
Then, we say that $\bPk$ is an $\varepsilon$-\defn{stabilizer} with respect to $\repspace$ and $\concepts$ if $\MIq(\rvX; \rvH \mid \rvC) - \MIq(\rvX; \rvHerase \mid \rvC) < \varepsilon$.\looseness=-1
\end{defin}
The data processing inequality once again ensures that $\MIq(\rvX ; \rvH \mid \rvC) - \MIq(\rvX ; \rvHerase \mid \rvC) \geq 0$. 
Containment and stability together characterize the \emph{preservation} of information not related to concepts.\looseness=-1


\section{A Causal Graphical Model} \label{sec:causal}
We now propose a causal structure by which language models leverage concepts, in the form of a latent variable, in the generation process.
We relate this causal structure to the information partition definitions given in \cref{sec:functionalsubspace}.
This enables causal controlled generation via a do-intervention \cite{pearl2009causal} on the concept random variable $\rvC$.
We finish with a discussion of how our causal controlled generation approach improves upon existing approaches. 
\looseness=-1

\subsection{Concept as a Latent Variable}

We illustrate the traditional autoregressive causal structure, based on the model definition put forth in \cref{sec:lms}, in \cref{fig:causal-graph-a}.
The $\kleene{\alphabet}$-valued random variable $\rvbX$ represents the textual context that was previously sampled from the model, $\rvH$ is the deterministic contextual representation, and $\rvX$ the word which is sampled using $\rvH$.

To enable controlled generation with respect to the concept, we introduce a $\concepts$-valued latent variable $\rvC$ in the generation process, as shown in \cref{fig:causal-graph-b}.
We make two assumptions about $\rvC$.
First, we assume that the distribution of $\rvC$ is influenced by the textual context $\rvbX$, and, moreover, that $\rvC$ is not \emph{fully} determined by the context $\cxt$, i.e., $\rvC$ is stochastic.
This assumption is justified by the fact that the concept value of the next word may not be fully determined by the preceding context, as discussed in \cref{sec:lmsconcepts}.
Second, we assume that the concept is determined \emph{before} the word is sampled.
This enables controlled generation, as the concept can directly influence the sampled word $\word$.
In doing so, we break away from $\iota$, which deterministically assigned a concept value to a word based on the preceding context. 

Our two assumptions on $\rvC$ have an important implication: $\rvbX$ is no longer the only source of stochasticity in $\rvH$, as in \cref{fig:causal-graph-a}.
Rather, we assume that both $\rvbX$ as well as $\rvC$ influence the representation $\rvH$, i.e., $\bh = \bh\left(\cxt, \concept\right)$.
Although this construction is not the norm in neural language models, it is a minor departure from reality that greatly enables our model.

We note that our causal structure in \cref{fig:causal-graph-b} differs from the high-level causal abstraction proposed by \citet{geiger2023causal} for concept erasure via linear projection in the representation space. 
The authors don't include a concept-valued random variable in their causal abstraction. Relatedly, they argue that iterative null space projection \citep[INLP,][]{ravfogel-etal-2020-null} attempts to determine whether a concept is used by a model, not how it is used.
In \cref{sec:intro}, we contend that language models rely on some notion of linguistic concepts like \conceptf{verbal-number}, since they consistently predict the correct value. 
Our intrinsic information-theoretic framework helps us identify both whether \emph{and} how a concept encoding is used by a language model because we measure changes in the model's predictions when projecting representations onto the concept and non-concept subspaces.

\subsection{Causal Controlled Generation} \label{sec:causal-con-gen}

We now derive a formal relationship between erasure, encapsulation, stability, containment, and the assumed causal graph in \cref{fig:causal-graph-b}. 
First, inspecting \cref{fig:causal-graph-b}, we see that if we wish to intervene on $\rvC$ to influence $\rvX$, there is a single \backdoorcolor{backdoor path} from $\rvC$ to $\rvH$.
As shown in \cref{fig:causal-graph-c}, \emph{intervening} on $\rvC$ directly (denoted by $\causaldo(\rvC = \concept)$) removes the edge $\rvbX \rightarrow \rvC$, which lets us easily compute the distribution over the next word after intervention as follows
\begin{align} \label{eq:do-intervention}
p(&\word \mid \rvHerase = \bhbot, \causaldo(\rvC = \concept)) \\
&= \sum_{\boldg \in \repspace}p(\word \mid \rvH = \bhbot + (\eyeminusPk)\boldg)\,p(\boldg \mid \concept) \nonumber
\end{align}
where, as shown in \cref{fig:causal-graph-b}, we assume that $\bhbot$ is deterministic given the context $\cxt$. $\boldg$ is an $\Rd$-valued contextual representation that encodes a textual context $\cxt'$ with concept value $\concept$. With high probability, $\cxtenc$ and $\boldg(\cxt')$ will be different. This is the logical conclusion of our decision to treat $\bhbot$ and $\bhpar$ as statistically independent---we can intervene on the generation process by setting the value of the concept component independently.\looseness=-1

We now make good on our decision to name the counterfactual unigram distribution from \cref{eq:counterfactualunigram} as such.
Assuming the model \cref{fig:causal-graph-b}, a do-intervention on $\rvC$---as depicted in \cref{fig:causal-graph-c}---implies erasure, encapsulation, stability, and containment.
We make this idea formal in the following theorem.
\begin{restatable}{theorem}{bigboy}\label{thm:graph}
Consider a joint distribution $p$ that factors as in \cref{fig:causal-graph-b}, parameterized by orthogonal projection matrix $\bPk$.
Under the distribution\looseness=-1
\begin{align}
\pdo(&\word, \bhbot, \bhpar, \concept) = p(\word \mid \bhbot, \bhpar)\\
&p(\bhbot \mid \mathrm{do}\left(\rvC = \concept\right))\,p(\bhpar \mid \mathrm{do}\left(\rvC = \concept\right))\,p(\concept) \nonumber
\end{align}
we have that $\bPk$ is an $\varepsilon$-eraser, $\eyeminusPk$ is an $\varepsilon$-encapsulator, $\eyeminusPk$ is an $\varepsilon$-container and $\bPk$ is an $\varepsilon$-stabilizer for every $\varepsilon > 0$.\looseness=-1
\end{restatable}
\begin{proof}
See \cref{app:graph}.
\end{proof}
What \Cref{thm:graph} tells us is that the graph given in \cref{fig:causal-graph-b} is consistent with the technical elaboration in \cref{sec:functionalsubspace}. 
Specifically, it means that erasure, encapsulation, stability, and containment are all properties that we expect a causal distribution resulting from an intervention on a concept to have. 
The interventional distributions, hence, motivate our discussion on independent $p(\bhpar)$ and $p(\bhbot)$ in \cref{sec:counterfactual}.\looseness=-1

\subsection{Non-causal Controlled Generation} \label{sec:noncausalgen}

Controlled generation involving the manipulation of concepts is not a new problem. 
We contextualize our approach relative to \citeposs{yang-klein-2021-fudge} method.
They perform controlled generation as follows. 
First, they train a classifier to predict a concept value $\concept \in \concepts$ from the contextual representation $\bh$ of a language model.
Then, they perform controlled generation by conditioning on a concept value $\rvC = \concept$ and applying Bayes' rule as follows:
\begin{align}
p(&\word \mid \cxt, \rvC = \concept) \\ 
&\propto  p(\rvC = \concept \mid (\eyeminusPk)\cxtenc)\,p(\word \mid \cxt) \nonumber
\end{align}
We illustrate the causal structure implied by this approach in \cref{fig:causal-graph-d}.
We use $\bPk$ to relate this approach to our subspace formulation,\footnote{Thus, we assume that the classifier is restricted to looking at $\rvHconcept$ to make its prediction.} but \citet{yang-klein-2021-fudge} do not make use of concept subspaces.

A do-intervention on $\rvC$ has no effect on $\rvX$ with this causal structure, because there is no causal path from $\rvC$ to $\rvX$ in \cref{fig:causal-graph-d}. 
This is why the authors \emph{condition} on $\rvC$ instead.
In this sense, \citeposs{yang-klein-2021-fudge} and similar methods are not causal and cannot easily be extended to be so.
As discussed in \cref{sec:causal-con-gen}, our approach \emph{is} causal, but such an analysis may come at the price of a number of restricting assumptions that are not fully met in practice. 
In the next section, we explain how we go about testing these assumptions with data.

\section{Experimental Setup} \label{sec:experimentsetup}

In the remainder of the paper, we test our framework empirically. Specifically, we answer three questions. First, are we able to find a projection matrix $\bPk$ that meets our definitions in \cref{sec:functionalsubspace}, across multiple concepts and models? Second, can we use the resulting concept subspace to successfully control the model's generation behavior, as theorized in \cref{sec:causal}? Finally, how does our do-intervention compare to other concept-based explanation methods on the CEBaB benchmark? 
\looseness=-1

\subsection{Linguistic Concepts} \label{sec:expdetailsling}

\paragraph{Concepts and Models.} 

We perform our analysis on two linguistic concepts, \conceptf{verbal-number} in English with $\concepts = \{\conceptval{sg}, \conceptval{pl}, \conceptval{n/a}\}$ and \conceptf{grammatical-gender} in French with $\concepts = \{\conceptval{fem}, \conceptval{msc}, \conceptval{n/a}\}$. 
For each of these concepts, we study the representation spaces of autoregressive language models, namely GPT2 \citep{radford2019language} and Llama 2 \cite{touvron2023llama}.\footnote{We rely on the implementations in the transformers library \citep{wolf-etal-2020-transformers}, namely: \modelnamef{gpt2-large} and \modelnamef{meta-llama/Llama-2-7b-hf} for \conceptf{verbal-number} and \modelnamef{gpt2-base-french} for \conceptf{grammatical-gender}.} 

\paragraph{Data.} 
For \conceptf{verbal-number} in English, we use \citeposs{linzen2016assessing} number agreement dataset. 
This dataset consists of sentences from Wikipedia that contain a \conceptval{sg} or \conceptval{pl} verb with the \textbf{fact} (ground truth verb) and the \textbf{foil} (inflected form of the fact to have opposite concept value). 
For \conceptf{grammatical-gender} in French, we rely on three treebanks from Universal Dependencies \citep{nivre-etal-2020-universal}: French GSD \citep{guillaumegsd}, ParTUT \citep{Sanguinetti2015, sanguinettipisa2014, boscotub2014}, and Rhapsodie \citep{lacheretrhapsodie2014}. 
We replicate the preprocessing steps of \citet{linzen2016assessing} on each of these datasets, i.e., we filter sentences to those containing \conceptval{fem} or \conceptval{msc} nouns with an associated adjective, and we obtain the foil by inflecting the \conceptf{grammatical-gender} of this adjective.

\subsection{CEBaB Benchmark}

The CEBaB \citep{abraham-etal-2022-cebab} benchmark dataset consists of original restaurant reviews in English, annotated for their overall sentiment on a 5-star scale, as well as aspect-level sentiment labels on four concepts, \conceptf{ambiance}, \conceptf{food}, \conceptf{noise}, and \conceptf{service}, with concept values $\concepts = \{\conceptval{pos}, \conceptval{neg}, \conceptval{n/a}\}$. The dataset includes human-written counterfactual reviews, where the text of the original review is altered to create pairs of reviews that differ according to a single aspect-level concept. 
We use this dataset in two experiments requiring different setups. In this section, we outline how we adapt the dataset to be able to learn $\bPk$ and compute our information-theoretic framework on the four CEBaB concepts. Later, in \cref{sec:cebabresults}, we also test our do-intervention on the CEBaB benchmark. For that experiment, we refer the reader to \citet{abraham-etal-2022-cebab} for experiment details, since we simply plugged our do-intervention into their pipeline.

\paragraph{Data and Models.} For CEBaB aspect-level sentiment concepts, we use the pre-processing steps in the CEBaB repository\footnote{\modelnamef{https://github.com/CEBaBing/CEBaB}}, yielding train-exclusive, dev and test data splits. 
To adapt CEBaB to an autoregressive language modeling setting, we append a concept-specific prompt at the end of the restaurant review, e.g., \wordf{... The food was}.
This prompt induces the model into predicting either a \conceptval{pos} or \conceptval{neg} adjective depending on the contents of the review with respect to that concept.\footnote{Initially, we learned $\bPk$ without the prompt, using the last hidden state of the review to replicate \citeposs{abraham-etal-2022-cebab} training process for $\bPk$ using INLP \citep{ravfogel-etal-2020-null}. 
Poor erasure performance led us to use prompts for both learning and evaluation.}
We randomly sample one of five concept-related prompts for each review and append it to the end of the review text (see \Cref{tab:app-cebab-prompts} in \Cref{app:conceptwordsprompts} for prompts).
Then, we take the last hidden state of the prompt to learn $\bPk$ and compute our counterfactual mutual information.
This prompt-based approach enables the use of our framework for a wide range of complex concepts.
Lastly, we study these four concepts using GPT2 \citep{radford2019language} and Llama 2 \cite{touvron2023llama}, using the same transformers implementations as for \conceptf{verbal-number} \citep{wolf-etal-2020-transformers}.

\begin{table*}
    \centering
    \small  
    \resizebox{0.99\textwidth}{!}{%
    \begin{tabular}{l l C{0.1\textwidth} C{0.1\textwidth} C{0.1\textwidth} C{0.1\textwidth} C{0.1\textwidth} C{0.1\textwidth} C{0.1\textwidth} C{0.1\textwidth}}
         \toprule
         & & & \multicolumn{4}{c}{Concept Information} & \multicolumn{3}{c}{Non-Concept Information}  \\ \cmidrule(lr){4-7} \cmidrule(lr){8-10} 
         Concept & Model &  $\MI(\rvC; \rvH)$ & Erasure   ($\uparrow$) & Corr. Erasure   ($\uparrow$) & Encaps.   ($\uparrow$) & Reconst.   ($\uparrow$) &  {\scriptsize $\MI(\rvX\!;\! \rvH\! \mid\! \rvC)$} & Containment   ($\uparrow$) & Stability   ($\uparrow$) \\ \toprule
            \conceptf{number} & \modelnamef{gpt2-large} & 0.50\pms{0.04} & 0.78\pms{0.04} & 0.69\pms{0.04} & 0.52\pms{0.03} & 0.74\pms{0.04} & 1.02\pms{0.15} & 0.87\pms{0.02} & 1.00\pms{0.02} \\
            \conceptf{number} & \modelnamef{llama2} & 0.49\pms{0.06} & 0.76\pms{0.05} & 0.75\pms{0.04} & 0.55\pms{0.05} & 0.78\pms{0.08} & 1.07\pms{0.07} & 0.78\pms{0.01} & 1.05\pms{0.02} \\
            \midrule
            \conceptf{gender} & \modelnamef{gpt2-base\textsubscript{fr}} & 0.46\pms{0.04} & 0.57\pms{0.04} & 0.49\pms{0.08} & 0.34\pms{0.03} & 0.77\pms{0.03} & 2.02\pms{0.13} & 0.93\pms{0.01} & 0.95\pms{0.01} \\
            \midrule
            \conceptf{ambiance} & \modelnamef{gpt2-large} & 0.25\pms{0.02} & 0.84\pms{0.04} & 0.97\pms{0.01} & 0.09\pms{0.02} & 0.25\pms{0.04} & 0.41\pms{0.10} & 0.77\pms{0.07} & 0.76\pms{0.13} \\
            \conceptf{ambiance} & \modelnamef{llama2} & 0.35\pms{0.03} & 0.39\pms{0.04} & 0.18\pms{0.05} & 0.46\pms{0.03} & 1.07\pms{0.06} & 0.58\pms{0.13} & 0.72\pms{0.06} & 0.84\pms{0.07} \\
            \conceptf{food} & \modelnamef{gpt2-large} & 0.28\pms{0.02} & 0.74\pms{0.04} & 0.27\pms{0.12} & 0.59\pms{0.04} & 0.85\pms{0.06} & 0.44\pms{0.09} & 0.81\pms{0.03} & 0.60\pms{0.09} \\
            \conceptf{food} & \modelnamef{llama2} & 0.41\pms{0.03} & 0.76\pms{0.03} & 0.68\pms{0.03} & 0.97\pms{0.07} & 1.21\pms{0.06} & 0.66\pms{0.07} & 0.73\pms{0.03} & 0.81\pms{0.05} \\
            \conceptf{noise} & \modelnamef{gpt2-large} & 0.21\pms{0.01} & 0.78\pms{0.01} & 0.60\pms{0.03} & 0.35\pms{0.05} & 0.57\pms{0.05} & 0.31\pms{0.11} & 0.71\pms{0.14} & 0.82\pms{0.10} \\
            \conceptf{noise} & \modelnamef{llama2} & 0.28\pms{0.03} & 0.29\pms{0.08} & 0.21\pms{0.10} & 0.43\pms{0.08} & 1.15\pms{0.05} & 0.60\pms{0.10} & 0.73\pms{0.05} & 0.73\pms{0.09} \\
            \conceptf{service} & \modelnamef{gpt2-large} & 0.35\pms{0.03} & 0.65\pms{0.03} & 0.40\pms{0.11} & 0.44\pms{0.05} & 0.79\pms{0.05} & 0.39\pms{0.11} & 0.74\pms{0.08} & 0.74\pms{0.14} \\
            \conceptf{service} & \modelnamef{llama2} & 0.39\pms{0.02} & 0.51\pms{0.04} & 0.25\pms{0.05} & 0.86\pms{0.07} & 1.35\pms{0.09} & 0.53\pms{0.13} & 0.68\pms{0.08} & 0.86\pms{0.05} \\
             \bottomrule
    \end{tabular}%
    }
    \caption{Information-theoretic evaluation results. The first set of columns quantifies concept information in concept and non-concept subspaces, with total concept information $\MI(\rvC; \rvH)$ alongside the \textit{Correlational Erasure Ratio}, and the \textit{Counterfactual Erasure}, \textit{Encapsulation} and \textit{Reconstructed} ratios. The second set of columns describes the preservation of non-concept information, with total non-concept information $\MI(\rvX; \rvH \mid \rvC)$ alongside the \textit{Containment Ratio} and \textit{Stability Ratio}. We expect values of the ratios to fall within $[0,1]$, where higher is better. Each entry shows standard deviation over random restarts. The first set of rows shows linguistic concepts \conceptf{verbal-number} and \conceptf{grammatical-gender}, split by model. The second set of rows shows CEBaB aspect-level sentiment concepts \conceptf{ambiance}, \conceptf{food}, \conceptf{noise}, and \conceptf{service}, with $\bPk$ trained on the CEBaB dataset.}
    \label{tab:resultsmain}
    \vspace{-0pt}
\end{table*}

\subsection{Shared Experimental Setup} \label{sec:sharedsetup}

\paragraph{Concept Definition.} In \cref{sec:concepts}, we defined our context-dependent distribution $\iota$ as a means of relating language models and concepts. 
In practice, we drop the context-dependent aspect and define a concept via a list of words for each concept value. 
For linguistic concepts, we construct our lists of words by using SpaCy \citep{spacy} to tag the French and English Wikipedia corpora \citep{wikidump}, respectively. For \conceptf{verbal-number}, we use the tagged English words to obtain lists of third person present \conceptval{sg} and \conceptval{pl} verbs, which we then align to obtain matching pairs, e.g., (\wordf{walks}, \wordf{walk}). 
The process is the same for \conceptf{grammatical-gender} in French, leading to gendered pairs of adjectives, e.g., (\wordf{français}, \wordf{française}). 
For CEBaB concepts, we prompt ChatGPT \cite{chatgpt} to create lists of \conceptval{pos} and \conceptval{neg} adjectives relating to each concept. In this case, the lists do not lend themselves to pairing, but include equal numbers of words for each concept (see \Cref{tab:app-word-list-linguistic} for word lists). 

\paragraph{Computing word probabilities.}
A word can be tokenized to one or more tokens.
We assume that the tokenization for each word is unique. 
When computing the probability of a multi-token word while intervening on the representation space, e.g., $\qunigram(\word | \bhbot)$, we apply concept erasure only when computing the probability of the first token in the word. 
Subsequent token probabilities are obtained by recomputing $\bh$ without intervention.
We then sum log probabilities of each successive token, as well as the log probability that the token after the word is either the beginning of a new word, punctuation, or the $\eos$ token.
\looseness=-1


\paragraph{Finding the Concept Subspace.} 
We find $\bPk$ using LEACE \citep{belrose2023leace}, the state-of-the-art method for linear concept erasure. 
LEACE maximizes a cross-entropy loss on samples from $\ptildeunigram$ with respect to $\bPk$, which constitutes a lower bound on the \emph{correlational} $\MI(\rvC; \rvHerase)$. 
Results are reported for three $\bPk$ estimates obtained from randomized train, test splits for each concept, and three random restarts of the experiment for each $\bPk$.
\looseness=-1

\paragraph{Estimating $\punigram(\bh)$, $\punigram(\bhbot)$, and $\punigram(\bhpar)$.} For each model, we generate strings by repeatedly sampling from the language model using ancestral and nucleus sampling, starting with $\bos$, until the model's context size limit is reached or $\eos$ is sampled. 
At each sampling step, we save the pair $(\word_t, \cxtenc, \concept_t)$, where $\concept_t$ is the concept label of the sample. 
We then compute sums over $\repspace$ by randomly sampling from our dataset of generated $\cxtenc$ assuming a uniform distribution. 
We apply $\bPk$ to get $\bhbot$ and $\bhpar$.
\looseness=-1

\paragraph{The concept value $\conceptval{n/a}$.} 
We exclude $\conceptval{n/a}$ from our concept set when learning $\bPk$ and when measuring concept information. 
We are most interested in the behavior of the model when generating non-\conceptval{n/a} words. 
In the overwhelming majority of textual contexts, \conceptval{n/a} is far more likely and erasure is trivial.
As such, our train sets for LEACE and test sets for computing our metrics contain only non-\conceptval{n/a} context strings. 
This means $\bPk$ only learns to erase the distinction between non-\conceptval{n/a} concept values, so we again exclude \conceptval{n/a} when computing concept information metrics of erasure, encapsulation, and partition reconstruction defined in \cref{sec:erasure}.
We can, however, include \conceptval{n/a} words for non-concept information metrics of containment and stability defined in \cref{sec:containstab}. 
We estimate this by randomly sampling and computing the probabilities of $3$k non-concept words in the concept language with each random restart of our evaluation pipeline.
\looseness=-1


\section{Results} \label{sec:results}

\subsection{Partitioning of Concept Information} \label{sec:partitionresults}

In \Cref{tab:resultsmain}, we test empirically whether LEACE \citep{belrose2023leace} yields a $\bPk$ that performs well according to our counterfactual information-theoretic framework defined in \cref{sec:functionalsubspace}.
To facilitate the interpretation of our results, we reformulate our definitions into ratios with values between 0 and 1, such that a higher value is better:
\vspace{1pt}
\begin{center} 
 \small
    \begin{tabular}{ll} \toprule
    Ratio & Definition \\ \midrule
    Erasure & $1 - \frac{\MIq(\rvC; \rvHerase)}{\MI(\rvC; \rvH)}$ \\[6pt]
    Corr. Erasure & $1 - \frac{\MI(\rvC; \rvHerase)}{\MI(\rvC; \rvH)}$ \\[6pt]
    Encapsulation & $\frac{\MIq(\rvC; \rvHconcept)}{\MI(\rvC; \rvH)}$ \\[6pt]
    Reconstructed & $\frac{\MIq(\rvC; \rvHconcept) + \MIq(\rvC; \rvHerase)}{\MI(\rvC; \rvH)}$ \\[6pt]
    Containment & $1 - \frac{\MIq(\rvX; \rvHconcept | \rvC)}{\MI(\rvX; \rvH | \rvC)}$ \\[6pt]
    Stability & $\frac{\MIq(\rvX; \rvHerase | \rvC)}{\MI(\rvX; \rvH | \rvC)}$  \\ \bottomrule
    \end{tabular}
\end{center}\vspace{3pt}



\Cref{tab:resultsmain} shows that LEACE can yield an adequate representation space partitioning according to our framework, but does not drive the counterfactual erasure and encapsulation ratios to 1. 
The mean erasure ratio across all datasets and models is $0.64$ and the encapsulation ratio is $0.50$.
This indicates that (a) the non-concept representation components $\rvHerase$ still contain some information about the concept and (b) the concept representation components $\rvHconcept$ cannot fully determine the concept $\concept$. 
Empirically, it is difficult to determine whether these findings are due to the LEACE objective falling victim to spurious correlations, as described in \cref{sec:example}, or to the concept encoding being non-linear.
For CEBaB concepts, the correlational erasure ratio is typically far lower than the counterfactual one. 
This suggests that while spuriously correlated information remains in $\rvHerase$, the actual concept direction has been identified more accurately than a simple correlational analysis would imply.\looseness=-1

A containment ratio less than $1$ indicates that $\rvHconcept$ generally contains some non-concept information, with a mean value of $0.77$. 
Since all concepts are binary after dropping \conceptval{n/a}, the concept subspaces are one-dimensional. 
The amount of non-concept information contained in a one-dimensional subspace suggests that the concept features are highly correlated with non-concept features.
The stability ratio has a mean of $0.83$, confirming that LEACE does some damage to non-concept information in $\rvHerase$. We note, however, that for our simple linguistic concepts, this damage is minimal. \looseness=-1

\begin{figure}[!ht]
    \centering
    \includegraphics[width=\columnwidth,trim={0mm 2mm 0 2mm},clip]{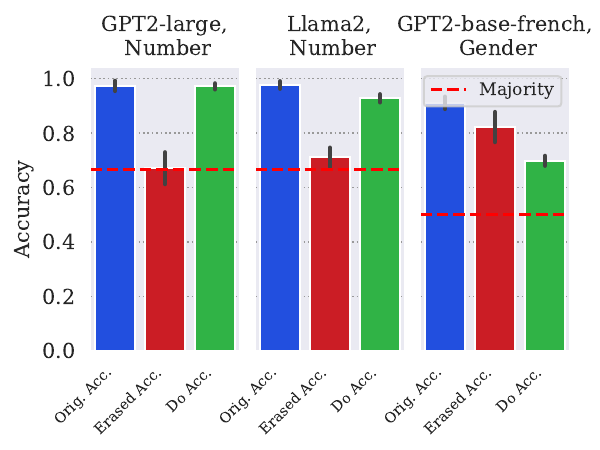}
    \caption{Controlled generation experiment. 
    Reported values are computed on (context, fact, foil) samples from the test split of our curated datasets of natural text used to train LEACE.
    \textit{Orig.\,\,Acc} refers to the accuracy with which the model chooses fact over foil using original representations. 
    \textit{Erased Acc.} is the accuracy after erasure, using our counterfactual $\qunigram(\word | \bhbot)$ distribution. 
    \textit{Do Acc.} measures, for example for Do(\rvC=\conceptvalfootnote{sg}) (see \cref{eq:do-intervention}), the rate at which the intervention induces the model to assign higher probability to the \conceptvalfootnote{sg} element of the (fact, foil) pair over its \conceptvalfootnote{pl} counterpart, reported on aggregate over \conceptvalfootnote{sg} and \conceptvalfootnote{pl} contexts in the test set.}
    \label{fig:resultscontrol}
     \vspace*{-0pt}
\end{figure}

Lastly, we attribute the failure to learn a concept partition for \conceptf{grammatical-gender} to limitations of the model itself.
Compared to English, the best available French \modelnamef{gpt2} model is trained on less data and has fewer parameters.
In our preliminary experiments, we noticed that smaller English \modelnamef{gpt2} models for \conceptf{verbal-number} were also notably worse than \modelnamef{gpt2-large}.

\begin{table*}[ht!]
    \centering
    \small  
    \resizebox{\textwidth}{!}{%
    \begin{tabular}{l l C{0.09\textwidth} C{0.09\textwidth} C{0.09\textwidth} C{0.09\textwidth} C{0.09\textwidth} C{0.09\textwidth} C{0.09\textwidth} C{0.09\textwidth} C{0.09\textwidth}}
         \toprule
         & & \multicolumn{3}{c}{Binary Overall Sentiment} & \multicolumn{3}{c}{3-way Overall Sentiment} & \multicolumn{3}{c}{5-way Overall Sentiment} \\ 
         \cmidrule(lr){3-5} \cmidrule(lr){6-8} \cmidrule(lr){9-11}
         Model & Method & ICaCE-cosine ($\downarrow$) & ICaCE-L2 ($\downarrow$) & ICaCE-normdiff($\downarrow$) & ICaCE-cosine ($\downarrow$) & ICaCE-L2 ($\downarrow$) & ICaCE-normdiff($\downarrow$) & ICaCE-cosine ($\downarrow$) & ICaCE-L2 ($\downarrow$) & ICaCE-normdiff($\downarrow$)\\
         \midrule
            \modelnamef{bert-base-uncased} & Best CEBaB & \textbf{0.64\pms{0.05}} & \textbf{0.31\pms{0.00}} & \textbf{0.30\pms{0.00}} & \textbf{0.54\pms{0.04}} & \textbf{0.56\pms{0.00}} & \textbf{0.48\pms{0.00}} & 0.63\pms{0.01} & \textbf{0.74\pms{0.02}} & \textbf{0.54\pms{0.02}} \\
            \modelnamef{bert-base-uncased} & INLP & 0.79\pms{0.00} & 0.52\pms{0.05} & 0.52\pms{0.05} & 0.70\pms{0.01} & 0.58\pms{0.02} & 0.55\pms{0.01} & 0.60\pms{0.02} & 0.80\pms{0.02} & 0.72\pms{0.03} \\
            \modelnamef{bert-base-uncased} & Do-INLP & 0.78\pms{0.01} & 0.55\pms{0.01} & 0.55\pms{0.01} & 0.68\pms{0.01} & 0.62\pms{0.04} & 0.53\pms{0.02} & \textbf{0.54\pms{0.01}} & 0.77\pms{0.02} & 0.61\pms{0.03} \\
            \modelnamef{bert-base-uncased} & LEACE & 0.77\pms{0.01} & 0.36\pms{0.01} & 0.35\pms{0.01} & 0.78\pms{0.03} & 0.57\pms{0.01} & 0.53\pms{0.01} & 0.77\pms{0.03} & 0.80\pms{0.02} & 0.71\pms{0.02} \\
            \modelnamef{bert-base-uncased} & Do-LEACE & 0.75\pms{0.01} & 0.37\pms{0.01} & 0.37\pms{0.01} & 0.69\pms{0.02} & 0.57\pms{0.01} & 0.49\pms{0.01} & 0.58\pms{0.02} & 0.75\pms{0.02} & 0.61\pms{0.03} \\
        \midrule
            \modelnamef{gpt2} & Best CEBaB & \textbf{0.58\pms{0.01}} & 0.29\pms{0.00} & 0.29\pms{0.00} & \textbf{0.50\pms{0.01}} & \textbf{0.52\pms{0.01}} & \textbf{0.42\pms{0.01}} & \textbf{0.59\pms{0.01}} & \textbf{0.60\pms{0.02}} & \textbf{0.40\pms{0.01}} \\
            \modelnamef{gpt2} & INLP & 1.00\pms{0.00} & 0.49\pms{0.04} & 0.44\pms{0.03} & 1.00\pms{0.00} & 0.65\pms{0.04} & 0.52\pms{0.01} & 1.00\pms{0.00} & 0.72\pms{0.02} & 0.58\pms{0.02} \\
            \modelnamef{gpt2} & Do-INLP & 0.98\pms{0.02} & 0.30\pms{0.01} & 0.30\pms{0.00} & 0.99\pms{0.02} & 0.53\pms{0.01} & 0.51\pms{0.01} & 0.99\pms{0.01} & 0.68\pms{0.02} & 0.66\pms{0.02} \\
            \modelnamef{gpt2} & LEACE & 1.00\pms{0.00} & 0.31\pms{0.01} & 0.30\pms{0.00} & 1.00\pms{0.00} & 0.53\pms{0.01} & 0.51\pms{0.01} & 1.00\pms{0.00} & 0.68\pms{0.02} & 0.66\pms{0.02} \\
            \modelnamef{gpt2} & Do-LEACE & 0.99\pms{0.06} & \textbf{0.28\pms{0.00}} & \textbf{0.28\pms{0.00}} & 0.99\pms{0.04} & \textbf{0.52\pms{0.01}} & 0.51\pms{0.01} & 0.99\pms{0.02} & 0.68\pms{0.02} & 0.67\pms{0.02} \\
        \midrule
            \modelnamef{roberta-base} & Best CEBaB & \textbf{0.70\pms{0.03}} & \textbf{0.29\pms{0.01}} & \textbf{0.29\pms{0.01}} & \textbf{0.62\pms{0.02}} & 0.55\pms{0.01} & 0.48\pms{0.00} & 0.64\pms{0.01} & \textbf{0.78\pms{0.01}} & \textbf{0.59\pms{0.01}} \\
            \modelnamef{roberta-base} & INLP & 0.80\pms{0.00} & 0.32\pms{0.02} & 0.32\pms{0.02} & 0.72\pms{0.01} & 0.55\pms{0.00} & 0.54\pms{0.01} & 0.59\pms{0.01} & 0.84\pms{0.01} & 0.81\pms{0.01} \\
            \modelnamef{roberta-base} & Do-INLP & 0.79\pms{0.00} & 0.30\pms{0.06} & 0.30\pms{0.06} & 0.70\pms{0.01} & \textbf{0.52\pms{0.01}} & 0.47\pms{0.02} & \textbf{0.56\pms{0.01}} & 0.80\pms{0.00} & 0.72\pms{0.01} \\
            \modelnamef{roberta-base} & LEACE & 0.82\pms{0.01} & 0.31\pms{0.01} & 0.31\pms{0.01} & 0.83\pms{0.01} & 0.54\pms{0.01} & 0.51\pms{0.01} & 0.83\pms{0.01} & 0.83\pms{0.01} & 0.80\pms{0.01} \\
            \modelnamef{roberta-base} & Do-LEACE & 0.77\pms{0.01} & 0.31\pms{0.01} & 0.31\pms{0.01} & 0.70\pms{0.01} & 0.53\pms{0.01} & \textbf{0.47\pms{0.01}} & 0.59\pms{0.01} & 0.81\pms{0.01} & 0.76\pms{0.01} \\
         \bottomrule
    \end{tabular}%
    }
    \caption{CEBaB benchmark results for our causal do-intervention. In three sets of columns, we report empirical Individual Causal Concept Effect (ICaCE) losses for 2-, 3- and 5-class overall sentiment prediction. Three types of losses are reported, lower is better for all three. \textit{ICaCE-cosine} indicates whether the estimated and observed effect have the same direction, without accounting for magnitude. \textit{ICaCE-L2} compares the Euclidean norm of the difference of the estimated and observed effect, and therefore reflects both magnitude and direction differences. \textit{ICaCE-normdiff} is the absolute difference between the Euclidean norms of the observed and estimated effects, thereby honing in on magnitude differences. Each entry shows mean $\pm$ standard deviation over random restarts. We report results for the best performing method from CEBaB \citep{abraham-etal-2022-cebab}, INLP \citep{ravfogel-etal-2020-null}, and LEACE \citep{belrose2023leace}. We apply our causal do-intervention from \cref{sec:causal-con-gen}, this time in a classification setting, using $\bPk$ returned by INLP and LEACE. See \cref{sec:cebabresults} for discussion of these results.}
    \label{tab:cebabresults}
\end{table*}

\subsection{Causal Controlled Generation} \label{sec:controlresults}

In \cref{sec:causal}, we argued for a causal structure for language generation that allows us to intervene on the concept-valued random variable $\rvC$. 
We now test this causal model empirically by computing the do-intervention in \cref{eq:do-intervention}.
We define success for the intervention using the forced-choice setup shown in \cref{ex:gramm,ex:ungramm}. 
For example, given a context with a \conceptval{sg} fact, we consider $\causaldo(\rvC = \conceptval{pl})$ successful if $p(\word \mid \rvHerase = \bhbot, \causaldo(\rvC = \conceptval{pl}))$ assigns higher probability to the \conceptf{pl} foil over \conceptf{sg} fact.

Results for this experiment are shown in \Cref{fig:resultscontrol}.
For context, we report the model's accuracy in the forced-choice setup before (\textit{Orig. Acc.}) and after (\textit{Erased Acc.}) erasure. 
We note the consistency of results between information-theoretic metrics in \Cref{tab:resultsmain} and post-erasure accuracy in \cref{fig:resultscontrol}---the erasure intervention successfully lowers the accuracy of the minority class \conceptval{pl} for \conceptf{verbal-number}, however the intervention fails to significantly reduce accuracy for \conceptf{grammatical-gender}.

With this context in mind, the do-intervention is remarkably successful for \conceptf{verbal-number} across two models. 
By acting solely in our concept subspace, we are able to almost match the original accuracy for both models. 
Results for \texttt{gpt2-base-french} are much worse, since the do-intervention actually has lower accuracy than erasure.
Viewed together with results in \cref{sec:partitionresults}, this confirms that our causal structure only holds given an adequate $\bPk$ under our counterfactual framework. 
Nonetheless, the success of the do-intervention on \conceptf{verbal-number} despite an imperfect partitioning of concept information suggests that identifying the causal concept direction is not a necessary requirement for causal concept-based controlled generation, so long as $\rvHconcept$ contains a significant share of concept information.\looseness=-1

\subsection{CEBaB Benchmark Results} \label{sec:cebabresults}

As explained in \cref{sec:related-work}, our work joins an ongoing conversation about whether interpretability methods such as linear concept erasure, which operate on low-level neural representations, are good estimators for a high-level causal narrative, e.g., the process by which a language model generates grammatical text \cite{geiger2023causal}.
The CEBaB benchmark \citep{abraham-etal-2022-cebab} evaluates concept-based explanation methods as causal effect estimators by comparing the change in a sentiment classifier's predictions resulting, e.g., from concept erasure, against a ground truth. 
This ground truth is the change in the classifier's predictions from changing the concept value in the input text, without modifying other aspects of the restaurant review.\looseness=-1

In \Cref{tab:cebabresults}, we report CEBaB results for INLP \citep{ravfogel-etal-2020-null} and LEACE \citep{belrose2023leace}, our do-intervention with the projection matrices returned by these methods, and the best performing score from other methods in CEBaB.
Results show that our do-intervention improves the performance of INLP and LEACE on CEBaB.
For LEACE, our do-intervention improves the directional (cosine) loss, but not the magnitude losses. 
The inverse is true for INLP.
Across models and tasks, our method occasionally beats other methods on CEBaB. 
\looseness=-1

\section{Related Work} \label{sec:related-work}

The use of concept subspaces for estimating causal effects and performing causal interventions is a well-studied area. 
\citet{feder-etal-2021-causalm} introduce a method for estimating the causal effect of a concept encoding on a language model's predictions, in a manner that accounts for confounders. 
They use adversarial training to obtain counterfactual representations with a different concept value, without changing the textual context.
The CEBaB benchmark \citep{abraham-etal-2022-cebab} builds on this work by providing a dataset of textual counterfactuals, enabling the comparison of different concept intervention methods.
\citet{arora-etal-2024-causalgym} also provide textual counterfactuals to benchmark a much larger set of linguistics tasks.
\citet{elazar-etal-2021-amnesic} and \citet{lasri-etal-2022-probing} measure changes in model accuracy on a concept-related task as a test of whether a model uses a linear concept encoding. 
\citet{jacovi-etal-2021-contrastive}, \citet{ravfogel-etal-2021-counterfactual}, \citet{park2023linear}, \citet{geiger2024das}, and \citet{wuandarora2024reft} create concept counterfactual representations using linear projection.
\citet{geiger2021causalabs, geiger2023causal} propose a theoretical framework relating low-level interpretability methods such as linear concept erasure to high level causal phenomena such as text generation.\looseness=-1

Previous work in linear concept erasure is also interested in measuring the degree to which erasure preserves non-concept-related features. 
\citet{ravfogel-etal-2020-null, ravfogel2022linear, ravfogelkernel2022} perform various tests, e.g., evaluating whether the semantics of the representation space were affected by erasure using SimLex-999 \citep{hill2015-simlex}, which is different from whether the language model's predictions have changed. 
\citet{elazar-etal-2021-amnesic} assess damage to $\plm$ via two tests: First, they test the model's ability to recover task performance after finetuning, and second, they report the overall KL divergence in the LM's output distribution, over the entire vocabulary. 
This last approach was a source of inspiration for our work, which delves much deeper into this distributional distance idea via our stability and containment tests.\looseness=-1


\section{Conclusion}
In this paper, we set out to define an \emph{intrinsic} measure of information in a subspace of a language model's representation space. 
In light of the correlational failure mode of linear concept erasure methods \cite{kumar2022probing}, doing so requires a counterfactual approach:
By assuming statistical independence between the components of a representation in the concept subspace and its orthogonal complement, we are able to correctly measure information in a subspace by marginalizing out the remainder of the space.
To the extent that a causal concept subspace exists for a particular concept and model, erasure under this metric is optimized by that subspace.
In practice, we did not actually optimize this metric, because it is computationally intractable due to nested sums over the infinite representations space $\repspace$. 
We leave the development of a tractable approximation to future work.
Our theoretical analysis, combined with the efficacy of linear erasure methods using a correlational objective, suggests a tantalizing prospect: That a counterfactual objective could identify a one-dimensional causal subspace containing \emph{all} information about the concept empirically.\looseness=-1


\bibliography{custom}
\bibliographystyle{acl_natbib}

\clearpage
\onecolumn

\appendix

\section{Decomposing $\MIq(\rvC; \rvH)$}\label{app:pleasantdecomposition}

\begin{restatable}{proposition}{pleasantdecomposition} \label{prop:pleasantdecomposition}
Suppose $\bPk$ is a $\varepsilon$-eraser and $(\eyeminusPk)$ is a $\varepsilon$-encapsulator of $\concepts$ with respect to $\repspace$.
Then, as $\varepsilon \rightarrow 0$, the following holds
\begin{equation} \label{eq:pleasantdecomposition}
\MIq(\rvC; \rvH) = \MIq(\rvC; \rvHerase) + \MIq(\rvC; \rvHconcept)
\end{equation}
\end{restatable}

\begin{proof}
On the left-hand side,
\begin{subequations}
\begin{align}  \MIq(\rvC; \rvH)  + \varepsilon &\ge \MIq(\rvC; \rvHconcept) + \varepsilon 
 & \justify{data-processing inequality} \\
&\ge \MIq(\rvC; \rvHconcept) + \MIq(\rvC; \rvHerase)  
 & \justify{$\bPk$ is an $\varepsilon$-eraser} \label{eq:pleasantdecomposition-lhs}
\end{align}
\end{subequations}
On the right-hand side,
\begin{subequations}
\begin{align}
\MIq(\rvC; \rvH) + \varepsilon &\le  \MIq(\rvC; \rvHconcept) +  2 \varepsilon&  \justify{$(\eyeminusPk)$ is an $\varepsilon$-encapsulator}\\
&\le \MIq(\rvC; \rvHconcept) + \MIq(\rvC; \rvHerase) + 2\varepsilon 
 & \justify{non-negativity of MI} \label{eq:pleasantdecomposition-rhs}
\end{align}
\end{subequations}
Combining \cref{eq:pleasantdecomposition-lhs} and \cref{eq:pleasantdecomposition-rhs}, we have
\begin{subequations} \label{eq:c-h-decomposition-sandwich}
\begin{align}
    \MIq(\rvC; \rvHconcept) + \MIq(\rvC; \rvHerase) &\le \MIq(\rvC; \rvH) + \varepsilon \\
    &\le \MIq(\rvC; \rvHconcept) + \MIq(\rvC; \rvHerase) + 2\varepsilon  
\end{align}
\end{subequations}
Taking $\varepsilon \to 0$ in \cref{eq:c-h-decomposition-sandwich}, we have \cref{eq:pleasantdecomposition}
\begin{subequations}
\begin{align}
     \MIq(\rvC; \rvH) = \MIq(\rvC; \rvHconcept) + \MIq(\rvC; \rvHerase) \nonumber
\end{align}
\end{subequations}
\end{proof}



\section{Proof of \Cref{thm:graph}}\label{app:graph}
\bigboy*

\begin{proof}
Given the factorization in \cref{fig:causal-graph-b}, we derive the following equation using the independence assumptions given in \cref{fig:causal-graph-b}:
\begin{subequations}
\begin{align}
\pdo(\word, \bhbot, \bhpar, \concept) 
&=p(\word \mid \bhbot, \bhpar)\,p(\bhbot \mid \mathrm{do}\left(\rvC = \concept\right))\,\pdo(\bhpar \mid \mathrm{do}\left(\rvC = \concept\right))\,p(\concept) &  \\
&= p(\word \mid \bhbot, \bhpar)
p(\bhbot)\,\pdo(\bhpar \mid \concept)\,p(\concept)  \label{eq:do-joint-derivation}
\end{align}
\end{subequations}

\paragraph{Erasure.}

Given \cref{eq:do-joint-derivation}, we have the following joint distribution
\begin{subequations}
\begin{align}
\pdo(\concept, \bhbot) &= \sum_{\bhpar \in \repspacepar} \sum_{\word \in \alphabet} \pdo(\word, \bhbot, \bhpar, \concept)   \\
&= \sum_{\bhpar \in \repspacepar} \sum_{\word \in \alphabet} p(\word \mid \bhbot, \bhpar)
p(\bhbot)\,\pdo(\bhpar \mid \concept)\,p(\concept) \\
&= \sum_{\bhpar \in \repspacepar} \underbrace{\left(\sum_{\word \in \alphabet} p(\word \mid \bhbot, \bhpar) \right)}_{=1}
p(\bhbot)\,\pdo(\bhpar \mid \concept)\,p(\concept)  \\
&= \underbrace{\left(\sum_{\bhpar \in \repspacepar} 
\pdo(\bhpar \mid \concept) \right)}_{=1}\,p(\bhbot)\,p(\concept) &  \\
&= p(\bhbot)\,p(\concept) & \label{eq:do-bot-c-independence}
\end{align}
\end{subequations}
The mutual information $\MI(\rvC ; \rvHerase)$ can be computed as follows 
\begin{subequations}
\begin{align}
    \MI(\rvC ; \rvHerase) &= \sum_{\concept \in \concepts}\sum_{\bhbot \in \repspacebot} \pdo(\concept, \bhbot) \log \frac{\pdo(\concept, \bhbot)}{p(\concept) p(\bhbot)} & \\
    &= \sum_{\concept \in \concepts}\sum_{\bhbot \in \repspacebot} \pdo(\concept, \bhbot) \log \frac{p(\bhbot)\,p(\concept)}{p(\concept) p(\bhbot)} & \justify{applying \cref{eq:do-bot-c-independence}}  \\
    &= 0 < \varepsilon \label{eq:graph-erasure-proof} &
\end{align}
\end{subequations}
for every $\varepsilon > 0$.

\paragraph{Encapsulation.}
The following equation holds given \cref{eq:do-joint-derivation}
\begin{subequations}
\begin{align}
     \MI(\rvC; \rvH) - \MI(\rvC; \rvHconcept) &= \MI(\rvC; \rvHconcept,\rvHerase ) - \MI(\rvC; \rvHconcept) &  \justify{$\rvH = \rvHerase, \rvHconcept$} \\
    &= \MI(\rvC; \rvHerase \mid \rvHconcept ) &  \\
    &= \MI(\rvC; \rvHerase ) & \justify{$\rvHerase, \rvHconcept$ are independent (\cref{sec:counterfactual})}  \\
    &= 0 < \varepsilon & \justify{applying \cref{eq:graph-erasure-proof}}   \\
\end{align}
\end{subequations}

\paragraph{Containment.}
The following joint distribution can be derived from \cref{eq:do-joint-derivation}
\begin{subequations}
\begin{align}
 \pdo(& \word, \bhpar, \concept = \concept)  
=  \sum_{\bhbot \in \repspacebot} \pdo(\word, \bhbot, \bhpar, \concept = \concept)  & \\
=&  \sum_{\bhbot \in \repspacebot} p(\word \mid \bhbot, \bhpar)
p(\bhbot)\,\pdo(\bhpar \mid \concept = \concept)\,p(\concept = \concept)  & \\
=&  \sum_{\bhbot \in \repspacebot} \frac{p(\word, \bhbot, \bhpar)}{p(\bhbot, \bhpar)}
p(\bhbot)\,\pdo(\bhpar \mid \concept = \concept)\,p(\concept = \concept)  & \\
=& \sum_{\bhbot \in \repspacebot} \frac{p(\word, \bhbot, \bhpar)}{\cancel{p(\bhbot)}\,p(\bhpar)} \cancel{p(\bhbot)} \,\pdo(\bhpar \mid \concept = \concept)\,p(\concept = \concept)  & \justify{$\rvHerase, \rvHconcept$ are independent (\cref{sec:counterfactual})}  \\
=& \underbrace{\sum_{\bhbot \in \repspacebot} p(\word,\bhbot \mid \bhpar)}_{=p(\word \mid \bhpar)} \,\pdo(\bhpar \mid \concept = \concept)\,p(\concept = \concept)  &  \\
=& \, p(\word \mid \bhpar) \,\pdo(\bhpar \mid \concept = \concept)\,p(\concept = \concept)  &  \\
=& \, p(\word \mid \bhpar, \concept=\concept) \,p(\bhpar \mid \concept=\concept)  &  \justify{$\rvHconcept$ is deterministic given $\rvC$ } \label{eq:do-par-x-independence}
\end{align}
\end{subequations}
The mutual information $\MI(\rvX; \rvHconcept \mid \rvC = \concept)$ can be computed as follows 
\begin{subequations}
\begin{align}
&    \MI(\rvX; \rvHconcept \mid \rvC = \concept) \\
    &= \sum_{\word \in \alphabet} \sum_{\bhpar \in \repspacepar} \pdo(\word, \bhpar,\concept=\concept) \log \frac{\pdo(\word, \bhpar,\concept=\concept)}{p(\word \mid \concept=\concept)  p(\bhpar|\concept=\concept)} &  \\
    &= \sum_{\word \in \alphabet} \sum_{\bhpar \in \repspacepar} \pdo(\word, \bhpar,\concept=\concept) \log \frac{p(\word \mid \bhpar, \concept=\concept) \,p(\bhpar \mid \concept=\concept)}{p(\word \mid \concept=\concept) \,p(\bhpar \mid \concept=\concept)} & \justify{applying \cref{eq:do-par-x-independence}}  \\
    &= \sum_{\word \in \alphabet} \sum_{\bhpar \in \repspacepar} \pdo(\word, \bhpar,\concept=\concept) \log \frac{p(\word \mid \bhpar, \concept=\concept) \,p(\bhpar \mid \concept=\concept)}{p(\word \mid \bhpar , \concept=\concept) \,p(\bhpar \mid \concept=\concept)} & \justify{$\rvHconcept$ is deterministic given $\rvC$}  \\
    &= 0 < \varepsilon  & \label{eq:graph-containment-proof}
\end{align}
\end{subequations}

\paragraph{Stability.}
The following equation holds given \cref{eq:do-joint-derivation}
\begin{subequations}
\begin{align}
    & \MI( \rvX; \rvH \mid \rvC = \concept) - \MI(\rvX; \rvHerase \mid \rvC = \concept) & \\ = 
    &\MI(\rvX; \rvHerase, \rvHconcept \mid \rvC = \concept) - \MI(\rvX; \rvHerase \mid \rvC = \concept) & \justify{$\rvH = (\rvHerase, \rvHconcept)$}  \\
    = &\MI(\rvX; \rvHconcept \mid \rvHerase, \rvC = \concept) & \justify{conditional mutual information}  \\
    = &\MI(\rvX; \rvHconcept \mid \rvC = \concept) & \justify{$\rvHerase, \rvHconcept$ are independent (\cref{sec:counterfactual})}  \\
    = &0 < \varepsilon & \justify{applying \cref{eq:graph-containment-proof}}  
\end{align}
\end{subequations}

\end{proof}

\section{Concept Word Lists and CEBaB Prompts} \label{app:conceptwordsprompts}

\begin{table}[!h]
    \centering \small
    \begin{tabular}{cl} \toprule
       \multirow{5}*{Ambiance}  & The ambiance was \\
         & The atmosphere was \\
         & The restaurant was \\
         & The vibe was \\
         & The setting was \\ \midrule
       \multirow{5}*{Food}  & The cuisine was \\
         & The dishes were \\
         & The meal was \\
         & The food was \\
         & The flavors was \\ \midrule
       \multirow{5}*{Noise}  & The ambient noise level was \\
         & The background noise was \\
         & The surrounding sound was \\
         & The auditory atmosphere was \\
         & The ambient soundscape was \\ \midrule
       \multirow{5}*{Service}  & The service was \\
         & The staff was \\
         & The hospitality extended by the staff was \\
         & The waiter was \\
         & The host was \\ \bottomrule
    \end{tabular}
    \caption{List of prompts for CEBaB dataset.}
    \label{tab:app-cebab-prompts}
\end{table}

\begin{table}[!h]
    \centering \small
    \begin{tabular}{ccp{10cm}} \toprule
       \multirow{2}*{\conceptf{verbal-number}} & \conceptf{sg} & absorbs, accepts, accompanies, accounts, achieves, acknowledges, activates, adds, addresses, administers, admits, adopts, advises, advocates, affects, agrees, aims, allows, announces, appears, applies, appoints, \textellipsis \\
        &  \conceptf{pl}  & absorb, accept, accompany, account, achieve, acknowledge, activate, add, address, administer, admit, adopt, advise, advocate, affect, agree, aim, allow, announce, appear, apply, appoint, \textellipsis \\ \midrule
       \multirow{2}*{\conceptf{grammatical-gender}} &  \conceptf{msc}  & abbatial, absolu, actif, actuel, additionnel, administratif, afro-américain, agressif, aigu, algérien, allemand, alsacien, amer, américain, ancien, annuel, architectural, arménien, artificiel, artisanal, \textellipsis \\
        &  \conceptf{fem}  & abbatiale, absolue, active, actuelle, additionnelle, administrative, afro-américaine, agressive, aiguë, algérienne, allemande, alsacienne, amère, américaine, ancienne, annuelle, architecturale, arménienne, artificielle, artisanale, \textellipsis \\ \midrule
    \end{tabular}
    \caption{Subset of linguistic concept word lists.}
    \label{tab:app-word-list-linguistic}
\end{table}

\begin{table}[!b]
    \centering \small
    \begin{tabular}{ccp{12cm}} \toprule
       \multirow{2}{*}{\conceptf{ambiance}} & \conceptf{pos} & cozy, elegant, inviting, charming, welcoming, intimate, sophisticated, tranquil, lively, romantic, chic, rustic, vibrant, serene, stylish, eclectic, enchanting, upscale, warm, bustling, idyllic, exquisite, radiant, harmonious, blissful, alluring, picturesque, opulent, sumptuous, dreamy, luxurious, polished, effervescent, enthralling \\
        & \conceptf{neg}  & dingy, claustrophobic, dreary, uninviting, dull, sterile, disorganized, loud, cramped, stale, unpleasant, chaotic, gaudy, tacky, uncomfortable, grimy, stuffy, depressing, drab, cold, seedy, pretentious, overcrowded, gloomy, oppressive, grim, tense, repellent, muggy, sullen, bland, repugnant, dismal, shabby \\ \midrule
       \multirow{2}{*}{\conceptf{food}} & \conceptf{pos} & delicious, mouthwatering, flavorful, delectable, savory, scrumptious, tasty, heavenly, exquisite, succulent, aromatic, satisfying, appetizing, divine, gourmet, nutritious, fresh, yummy, fragrant, sumptuous, delightful, rich, zesty, indulgent, juicy, decadent, balanced, sweet, tangy, refreshing, warm, tender, crispy, herbacious, spiced \\
        & \conceptf{neg}  & tasteless, bland, overcooked, undercooked, stale, soggy, greasy, unappetizing, flavorless, dry, tough, burnt, rancid, unpalatable, watery, disgusting, sour, bitter, mushy, unpleasant, unappealing, foul, off-putting, insipid, rubbery, dull, moldy, spoiled, repulsive, stinky, unbalanced, salty, fatty, stringy, oily \\ \midrule
       \multirow{2}{*}{\conceptf{noise}} & \conceptf{pos} & vibrant, lively, energetic, buoyant, festive, animated, cheerful, convivial, invigorating, buzzy, jubilant, pulsating, tranquil, serene, calm, peaceful, quiet, relaxed, soothing, gentle, mellow, harmonious, rejuvenating, vivacious, resonant \\
        & \conceptf{neg}  & disruptive, deafening, chaotic, clamorous, unruly, boisterous, raucous, overwhelming, harsh, grating, jarring, unpleasant, discordant, intrusive, irritating, nerve-wracking, agitated, distracting, unbearable, silent, loud, obnoxious, booming, cacaphonous, blaring \\ \midrule
       \multirow{2}{*}{\conceptf{service}} & \conceptf{pos} & attentive, friendly, efficient, professional, courteous, prompt, welcoming, accommodating, hospitable, personable, polished, gracious, knowledgeable, warm, engaging, diligent, exemplary, seamless, outstanding, enthusiastic, meticulous, personalized, considerate, anticipatory, respectful, reliable, consistent, thoughtful, empathetic, genuine, polite, proactive, adaptable, detail-oriented, impeccable \\
        & \conceptf{neg}  & inattentive, slow, rude, incompetent, dismissive, disorganized, impersonal, unprofessional, neglectful, indifferent, abrupt, inefficient, aloof, uncaring, clueless, forgetful, disrespectful, aggressive, insubordinate, inept, unfriendly, unaccommodating, inconsiderate, arrogant, sloppy, unknowledgeable, intrusive, negligent, careless, unresponsive, unreliable, distracted, impolite, disinterested, discourteous \\ \bottomrule
    \end{tabular}
    \caption{Full list of CEBaB concept words.}
    \label{tab:app-word-list-cebab}
\end{table}

\end{document}